\documentclass[11pt]{article}

\usepackage[utf8]{inputenc}
\usepackage[T1]{fontenc}
\usepackage{lmodern}

\usepackage{amsmath}
\usepackage{amssymb}
\usepackage{amsfonts}
\usepackage{mathtools}
\usepackage{amsthm}

\usepackage{graphicx}
\usepackage{booktabs}
\usepackage{array}
\usepackage{adjustbox}
\usepackage{enumitem}
\usepackage{microtype}
\usepackage{natbib}
\usepackage{xcolor}
\usepackage{url}
\usepackage{hyperref}
\graphicspath{{figures/}}
\newcommand{\arxivtablesetup}{\footnotesize\setlength{\tabcolsep}{4pt}\renewcommand{\arraystretch}{1.08}}
\newcolumntype{P}[1]{>{\raggedright\arraybackslash}p{#1}}

\newtheorem{theorem}{Theorem}

\newtheorem{corollary}[theorem]{Corollary}

\newcommand{\R}{\mathbb{R}}
\newcommand{\Pbb}{\mathbb{P}}
\newcommand{\Ebb}{\mathbb{E}}
\newcommand{\TV}{\operatorname{TV}}
\newcommand{\tr}{\operatorname{tr}}
\newcommand{\diag}{\operatorname{diag}}
\newcommand{\unif}{\operatorname{unif}}
\newcommand{\Pois}{\operatorname{Poisson}}
\newcommand{\rank}{\operatorname{rank}}

\newcommand{\norm}[1]{\lVert #1\rVert}
\newcommand{\ip}[2]{\langle #1,#2\rangle}

\hypersetup{
    colorlinks=true,
    linkcolor=blue,
    citecolor=blue,
    urlcolor=blue,
    pdftitle={Finite-Sample Spectral Calibration for Rank Inference in Foundation-Model Adapters},
    pdfauthor={Mohammed AHNOUCH and Lotfi Elaachak}
}

\title{Spectral Rank Certification for Foundation Model Adapters}

\author{
    Mohammed AHNOUCH
    \\
    Université Paris 1
    \\
    Paris, France
    \and
    Lotfi Elaachak
    \\
    Faculty of Science and Technology of Tangier
    \\
    Abdelmalek Essaadi University
    \\
    Tangier, Morocco
}

\date{}

\begin{document}

\maketitle

\begin{abstract}
Nominal LoRA rank is a design parameter; calibrated spectral evidence is a separate inferential quantity. This article develops a finite-sample framework for inferring effective rank structure in public foundation-model adapters. The theoretical core is an exact chi-square divergence for the fixed-dimensional Gaussian rank-one reference experiment, with an unknown signal direction integrated under a rotation-invariant reference prior. The resulting series yields a computable finite-sample Le Cam bound at concrete layer sizes, an explicit remainder bound for numerical truncation, and the rectangular Baik-Ben Arous-Peche (BBP) limit. A compact-manifold Laplace expansion shows that finite-sample likelihood evidence also depends on leading spectral gaps through the factor $s_1^{|m-n|}\prod_{i\ge2}(s_1^2-s_i^2)$, motivating joint calibration of clustered singular values. Building on these results, we introduce an empirical-null workflow for PEFT LoRA adapters: factor reconstruction, Monte Carlo $p$-values, stagewise and block testing, and module-wise and corpus-level BH reporting. In an audit of 26 public adapters, 684 modules, six architecture families, and 31,770 public-checkpoint spectra rows, calibrated effective rank is typically much smaller than nominal rank and differs systematically from 95\% energy retention. A measured RoBERTa-RTE slice on $n=24$ examples illustrates the measurement path from calibrated ranks to task evaluation, without treating the slice as a utility study. The main empirical finding is that calibrated effective rank is usually far below nominal rank, and that energy retention and statistical surprise answer different questions.
\end{abstract}


\section{Introduction}
Low-rank adaptation (LoRA) replaces a dense fine-tuning update by a low-rank adapter \citep{hu2022lora}. The chosen rank is usually an engineering hyperparameter: it controls trainable parameters and memory, while the number of trained components distinguishable from null structure is an inferential question. Two adapters with the same nominal rank can have very different spectra, and two truncation rules with similar reconstruction energy can have different false-positive interpretations. Public repositories of trained adapter factors make this distinction concrete: each checkpoint stores structured matrix updates across layers and tasks, and each module-level spectrum records how fine-tuning mass is distributed. A purely descriptive spectrum is not enough, because the same singular value can have different evidential meaning under different layer sizes, residual scales, and null constructions.

Given an adapter update $\widehat\Delta_\ell$ at layer or module $\ell$, the target output is a reproducible label of the form ``component $j$ is surprising under null construction $\mathcal N_\ell$ at level $q$''. This calibrated decision complements pruning and rank-allocation methods by separating statistical evidence from downstream engineering tradeoffs.

The core reference experiment is
\begin{equation}
H_0:Y=\sigma G,\qquad
H_1:Y=\theta uv^\top+\sigma G,
\label{eq:model}
\end{equation}
where $G_{ij}\stackrel{\mathrm{iid}}{\sim}N(0,1)$ and $u\sim\unif(S^{m-1})$, $v\sim\unif(S^{n-1})$ are independent. The Haar law supplies a rotation-invariant reference prior for an unknown signal direction. When a low false-positive detector has limited power against the average unknown direction in \eqref{eq:model}, singular-value thresholds require additional calibration before being interpreted as calibrated detections.

The contributions are threefold. First, the second moment between the null and Haar mixture is computed exactly at fixed $(m,n)$, yielding a finite-sample Le Cam bound, an explicit remainder bound for the series, and the rectangular BBP limit as a high-dimensional consequence \citep{baik2005phase,benaych2012singular}; a Laplace expansion further shows that the integrated likelihood depends on leading spectral gaps. Second, we give a deployable empirical-null workflow for adapter factors: PEFT reconstruction, Monte Carlo $p$-values, stagewise calibration, module-wise and corpus-level BH reporting, and block tests for clustered components. Third, a public-checkpoint audit shows that calibrated effective rank is typically far below nominal rank and that energy retention and statistical surprise answer different questions; a small RTE measurement-path illustration records retained-rank tradeoffs under the same truncation rules.

\section{Related Work}
LoRA and related PEFT methods reduce trainable parameters by constraining updates to a low-rank subspace \citep{hu2022lora}. During training, AdaLoRA reallocates rank budgets \citep{zhang2023adalora}, DoRA separates magnitude and direction in the weight update \citep{liu2024dora}, and VB-LoRA and LoRA-Mini explore alternative parameterizations and decompositions \citep{li2024vblora,singh2024loramini}. Those approaches answer how to allocate or reparameterize rank while learning. Once an adapter is trained, however, rank allocation is no longer the question: the issue is whether observed singular components are surprising under a stated null.

Spiked random-matrix models supply detection thresholds and outlier laws for low-rank perturbations \citep{baik2005phase,benaych2012singular,elalaoui2018rectangular}, and singular-value shrinkage with optimal hard thresholds studies how spectra separate signal from noise in an asymptotic MSE sense \citep{gavish2014optimal}. That asymptotic tradition motivates spectral rules, but it does not by itself supply a finite-sample false-positive bound at concrete layer dimensions; the exact chi-square series developed below does.

Because trained adapter residuals can be anisotropic, heavy-tailed, or block-structured, deployment also requires empirical nulls. Parallel analysis compares observed principal components with null spectra \citep{horn1965rationale,buja1992remarks}, while empirical-Bayes null estimation and resampling error control separate reference models from measured null behavior \citep{efron2010large,westfall1993resampling}. The add-one Monte Carlo $p$-value used below follows that resampling tradition. Benjamini-Hochberg (BH) remains a useful reporting rule \citep{benjamini1995controlling}, but after adaptive deflation its guarantees should be described as empirically calibrated unless the needed dependence conditions are proved.

\section{Exact Rank-One Bound}
For a selector $\phi(Y)\in\{0,1\}$ let $\alpha(\phi)=\Pbb_0\{\phi=1\}$ and $\pi(\phi;\theta)=\Pbb_1\{\phi=1\}$, where $P_1$ is the Haar mixture in \eqref{eq:model}. The mixture likelihood ratio is
\begin{equation}
L(Y)=\Ebb_{u,v}\exp\left\{\frac{\theta}{\sigma^2}\ip{Y}{uv^\top}-\frac{\theta^2}{2\sigma^2}\right\}.
\label{eq:lr}
\end{equation}

\begin{theorem}[Exact chi-square divergence]
Let $(u,v)$ and $(u',v')$ be independent draws from $\unif(S^{m-1})\times\unif(S^{n-1})$ and set $\lambda=\theta^2/\sigma^2$. Then
\begin{align}
\chi^2(P_1\Vert P_0)
&=\Ebb\exp\{\lambda\ip{u}{u'}\ip{v}{v'}\}-1 \notag\\
&=\sum_{j=1}^{\infty}\frac{\lambda^{2j}}{(2j)!}M_m(2j)M_n(2j),
\label{eq:chisq}
\end{align}
where, for independent $x,x'\sim\unif(S^{d-1})$,
\[
M_d(2j)=\Ebb\ip{x}{x'}^{2j}=\frac{(2j-1)!!}{d(d+2)\cdots(d+2j-2)}.
\]
\end{theorem}

\begin{proof}
Write $A=uv^\top$ and $A'=u'v'^\top$. Under $P_0$, conditional on $A,A'$, the Gaussian moment-generating function gives
\[
\Ebb_0 L(Y)^2=\Ebb_{A,A'}\exp\left\{-\theta^2/\sigma^2+\frac{\theta^2}{2\sigma^2}\norm{A+A'}_F^2\right\}.
\]
Because $\norm{A+A'}_F^2=2+2\ip{u}{u'}\ip{v}{v'}$, the first identity follows. Expanding the exponential is justified by boundedness. Odd terms vanish by symmetry, and independence separates the even moments. Rotational invariance fixes $x'=e_1$. The first coordinate $T=x_1$ of $x\sim\unif(S^{d-1})$ then has density
\[
f_T(t)=c_d(1-t^2)^{(d-3)/2},\qquad t\in[-1,1],
\]
with normalizing constant $c_d$ so that $\int_{-1}^1 f_T=1$. Even moments satisfy
\[
\Ebb T^{2j}=\frac{\int_0^1 t^{2j}(1-t^2)^{(d-3)/2}\,dt}{\int_0^1(1-t^2)^{(d-3)/2}\,dt}
=\frac{(2j-1)!!}{d(d+2)\cdots(d+2j-2)},
\]
which is the stated formula for $M_d(2j)$.
\end{proof}

The series is entire at fixed $m,n$ and admits an explicit remainder bound. Let $a_j$ be the $j$th term in \eqref{eq:chisq}, $S_K=\sum_{j\le K}a_j$, and $R_K=\chi^2-S_K$. For
$q_j=\lambda^2/[(m+2j)(n+2j)]$, if $q_{K+1}<1$ then
\begin{equation}
0\le R_K\le \frac{a_{K+1}}{1-q_{K+1}}.
\label{eq:tail}
\end{equation}
Indeed, $a_{j+1}/a_j=\lambda^2(2j+1)/[(2j+2)(m+2j)(n+2j)]\le q_j$, and $q_j$ decreases. The crude global bound $R_K\le e^{|\lambda|}\Pbb\{\Pois(|\lambda|)\ge 2K+2\}$ is also valid since $M_m(2j)M_n(2j)\le1$.

\begin{corollary}[BBP-scale limit]
If $m,n\to\infty$ and $\lambda^2/(mn)\to\beta\in[0,1)$, then
\[
\chi^2(P_1\Vert P_0)\to (1-\beta)^{-1/2}-1.
\]
Equivalently, if $\theta/[\sigma(mn)^{1/4}]\to c<1$, the limit is $(1-c^4)^{-1/2}-1$.
\end{corollary}
For fixed $j$, $M_m(2j)M_n(2j)=((2j-1)!!)^2(mn)^{-j}\{1+o(1)\}$, so the $j$th term converges to $\binom{2j}{j}(\beta/4)^j$. A geometric domination from \eqref{eq:tail} justifies summing, and the central-binomial generating function gives the limit.

The finite-sample testing consequence is
\begin{equation}
\pi(\phi;\theta)\le \alpha(\phi)+\frac12\sqrt{\chi^2(P_1\Vert P_0)}.
\label{eq:lecam}
\end{equation}
Indeed,
\[
\Ebb_1\phi-\Ebb_0\phi=\Ebb_0(L-1)\phi\le \TV(P_0,P_1)
\le \tfrac12\Ebb_0|L-1|
\le \tfrac12\sqrt{\Ebb_0(L-1)^2}
=\tfrac12\sqrt{\chi^2(P_1\Vert P_0)},
\]
where the last step is Cauchy-Schwarz on $L-1$ \citep{lecam1986asymptotic,tsybakov2009nonparametric}. Thus \eqref{eq:lecam} translates the divergence calculation into a bound on attainable power at a chosen false-positive rate for the reference experiment.

\begin{figure*}[t]
\centering
\begin{minipage}{.49\textwidth}
\centering
\includegraphics[width=\linewidth]{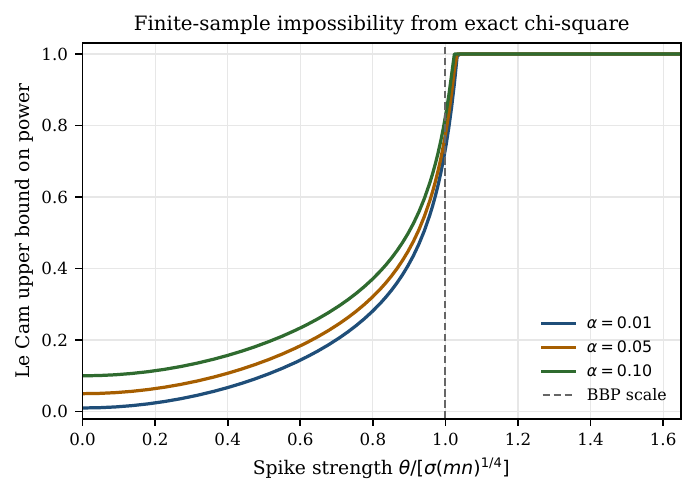}
\end{minipage}\hfill
\begin{minipage}{.49\textwidth}
\centering
\includegraphics[width=\linewidth]{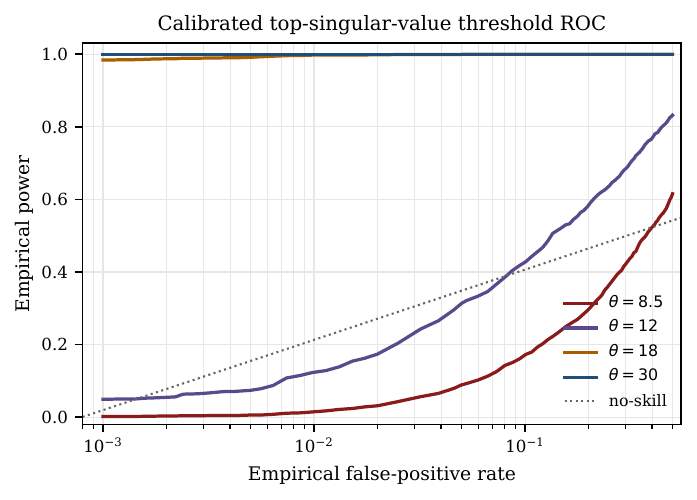}
\end{minipage}
\caption{Left: Le Cam power upper bounds from the exact chi-square series at $m=n=128$. Right: calibrated top-singular-value ROC curves from independent null and alternative simulations.}
\label{fig:calibration}
\end{figure*}

\section{Full-Spectrum Evidence and Higher-Rank Signals}
The unnormalized null edge is approximately $\sigma(\sqrt m+\sqrt n)$, while the additive rectangular outlier scale is $\sigma(mn)^{1/4}$ \citep{benaych2012singular}. For $m=n=d$, define $\widetilde Y=Y/(\sigma\sqrt d)$ and suppose $\theta/(\sigma\sqrt d)\to c$. The rank-one rectangular outlier theorem gives
\begin{equation}
s_1(\widetilde Y)\to \begin{cases}2, &0\le c\le1,\\ c+c^{-1}, &c>1,
\end{cases}
\label{eq:bbp}
\end{equation}
in probability. Thus a fixed asymptotic edge is a scale, not a finite-sample false-positive rate. A calibrated spectral rule should use the null quantile
\[
q_{1-\alpha}^{(0)}=\inf\{t:\Pbb_0(s_1(Y)\le t)\ge1-\alpha\},
\quad \phi_\alpha(Y)=\mathbf1\{s_1(Y)>q_{1-\alpha}^{(0)}\}.
\]

The likelihood ratio \eqref{eq:lr} is not a function of $s_1$ alone. Let
\[
I_\kappa(Y)=\int_{S^{m-1}}\int_{S^{n-1}} e^{\kappa u^\top Yv}
\,d\bar\omega_m(u)d\bar\omega_n(v),
\]
where $d\bar\omega_d$ is normalized surface measure. If $s_1>s_2$ and $r=\min(m,n)$, the Morse-Laplace theorem on compact manifolds \citep[Ch.~VIII]{wong2001asymptotic} gives
\begin{align}
\log I_\kappa(Y)
&=\kappa s_1-\frac{m+n-2}{2}\log\kappa+C(Y)+o(1),\notag\\
C(Y)&=\log2+\frac{m+n-2}{2}\log(2\pi)-\log(\omega_m\omega_n)\notag\\
&\quad-\frac12\log\left[s_1^{|m-n|}\prod_{i=2}^{r}(s_1^2-s_i^2)\right].
\label{eq:laplace}
\end{align}
The two maximizers are $(u_1,v_1)$ and $(-u_1,-v_1)$. In tangent coordinates $(x,y)$ around $(u_1,v_1)$, write $u=u_1\sqrt{1-\norm{x}^2}+Ux$ and $v=v_1\sqrt{1-\norm{y}^2}+Vy$ with orthonormal frames $U,V$ completing $u_1,v_1$. Expanding to quadratic order yields
\[
u^\top Yv=s_1-\frac{s_1}{2}(\norm{x}^2+\norm{y}^2)+\sum_{i=2}^{r}s_i x_i y_i+O(\norm{(x,y)}^3).
\]
The Hessian of $-u^\top Yv$ at the maximizer therefore has determinant
\[
\det(-H)=s_1^{|m-n|}\prod_{i=2}^{r}(s_1^2-s_i^2),
\]
which is exactly the gap factor in \eqref{eq:laplace}. Hence $I_\kappa$ and $s_1$ agree only at leading exponential order. Two matrices with the same $s_1$ can receive different integrated-likelihood evidence when their leading gaps differ. Near $s_1=s_2$, the nondegenerate expansion fails and a block statistic is more appropriate.

\begin{figure*}[t]
\centering
\begin{minipage}{.49\textwidth}
\centering
\includegraphics[width=\linewidth]{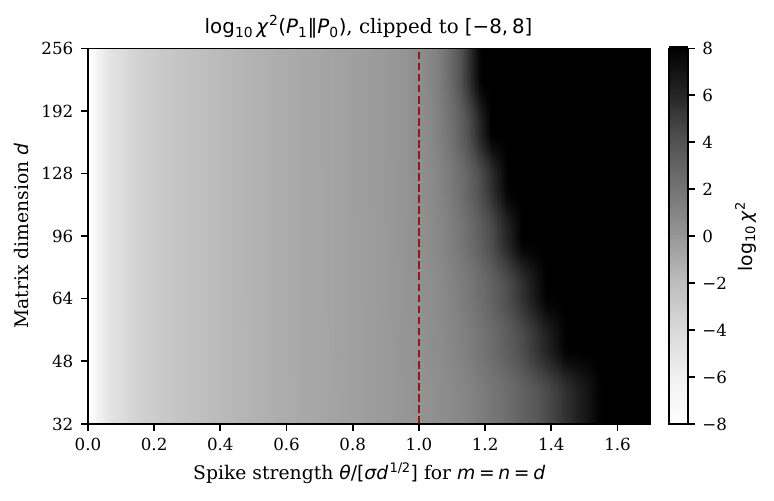}
\end{minipage}\hfill
\begin{minipage}{.49\textwidth}
\centering
\includegraphics[width=\linewidth]{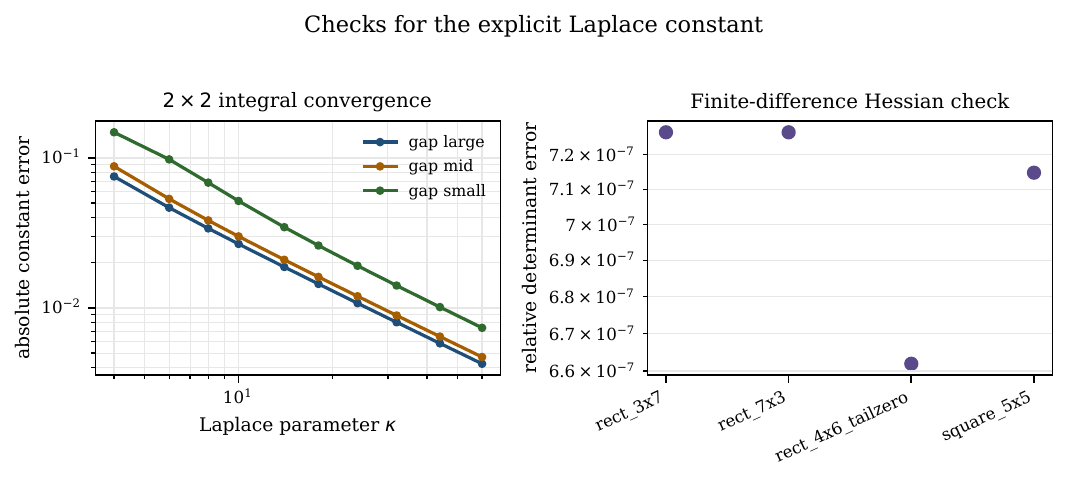}
\end{minipage}
\caption{Left: exact finite-sample $\log_{10}\chi^2(P_1\Vert P_0)$ across square dimensions, clipped to $[-8,8]$. Right: numerical checks of the Laplace constant and Hessian determinant in \eqref{eq:laplace}.}
\label{fig:chisq-laplace}
\end{figure*}

For a rank-$r$ signal $X_\star=U\Theta V^\top$ with independent Haar frames and $\Theta=\diag(\theta_1,\ldots,\theta_r)$,
\[
\chi^2_r+1=\Ebb\exp\{\sigma^{-2}\tr(\Theta U^\top U'\Theta V'^\top V)\}.
\]
If $Z=\tr(\Theta U^\top U'\Theta V'^\top V)$, then $\Ebb Z=0$, odd moments vanish by sign symmetry, and
\begin{equation}
\Ebb Z^2=\frac{\norm{\Theta}_F^4}{mn},\qquad
\chi^2_r=\frac{\norm{\Theta}_F^4}{2\sigma^4mn}+\mathcal R_4,
\label{eq:rankr}
\end{equation}
with $|\mathcal R_4|\le\sum_{k\ge4}(\norm{\Theta}_F^2/\sigma^2)^k/k!$. Expanding $\norm{\Theta}_F^4=\sum_i\theta_i^4+2\sum_{i<j}\theta_i^2\theta_j^2$ shows that the leading second-moment contribution includes the cross term
\[
\frac{2\theta_i^2\theta_j^2}{2\sigma^4mn}=\frac{\theta_i^2\theta_j^2}{\sigma^4mn}
\]
for each pair $(i,j)$. Componentwise BH tests each singular value marginally and therefore does not test that coupling; clustered near-critical components should be tested jointly (Section~\ref{sec:block}) or handled by a global rank budget.

\section{Empirical-Null Algorithms for Adapters}
The reference calculation becomes useful only after specifying a null. For trained adapters, the null can come from multiple fine-tuning seeds, residual-tail fitting, sign flips, factor/channel shuffles, permutations, or task-preserving bootstraps. The null construction must be reported because it is part of the scientific claim.

\subsection{PEFT Reconstruction}
Public PEFT-style LoRA checkpoints usually store adapter factors, from which dense updates can be reconstructed. For a module with input dimension $n$, output dimension $m$, and nominal rank $r$, the common convention is
\[
W_A=\texttt{lora\_A.weight}\in\R^{r\times n},\qquad
W_B=\texttt{lora\_B.weight}\in\R^{m\times r}.
\]
The dense adapter update is reconstructed as
\begin{equation}
\Delta W=\gamma\, W_B W_A,\qquad
\gamma=\alpha/r,
\label{eq:peft}
\end{equation}
or $\gamma=\alpha/\sqrt r$ when rank-stabilized LoRA is explicitly enabled. Consequently $\rank(\Delta W)\le r$ up to numerical tolerance. The statistical question is which trained singular components are surprising under a stated empirical null.

\subsection{Stagewise Testing}
The following procedure is the deployable version of the calibration layer. It is phrased as a reporting algorithm whose outputs include retained ranks, diagnostics, and abstentions when the chosen null construction falls outside calibration criteria.

\paragraph{Algorithm 1: Stagewise empirical-null testing.}
\begin{enumerate}[leftmargin=*,itemsep=2pt]
\item Inputs: adapter checkpoint or extracted factors, null constructor $\mathcal N$, target level $q$, Monte Carlo count $B$, random seed, and module list.
\item For each module $\ell$, reconstruct $\widehat\Delta_\ell$ using \eqref{eq:peft}; compute singular values $s_{\ell,1}\ge\cdots\ge s_{\ell,r_\ell}$.
\item Generate calibration null samples independently of any evaluation samples. For component $j$, compute the null statistic $T_{\ell,j}^{(b)}$ after applying the same deflation or block operation used on the observed adapter.
\item Use the add-one Monte Carlo value
\[
p_{\ell,j}=\frac{1+\sum_{b=1}^B \mathbf1\{T_{\ell,j}^{(b)}\ge T_{\ell,j}^{\rm obs}\}}{B+1}.
\]
\item Apply BH to the family of reported $(\ell,j)$ tests, or report fixed-level decisions. Unless dependence conditions are proved for the chosen pipeline, describe the BH output as empirically calibrated.
\item Output retained masks, effective ranks, $p$-values, uncertainty intervals, null diagnostics, and parameter-savings summaries.
\end{enumerate}

After data-dependent deflation, the residual is conditioned on the extracted subspace. The calibration target is therefore the full pipeline: draw null adapters, run deflation, compute all $p$-values, apply the same selection rule, and report realized false-positive and discovery summaries with Wilson or bootstrap intervals.

\subsection{Block Tests for Clustered Components}
\label{sec:block}
When leading gaps are small, the Laplace nondegeneracy condition $s_1>s_2$ fails and componentwise tests can overstate the distinction between adjacent singular directions. Algorithm~2 is the operational counterpart of that failure mode: it changes the unit of inference whenever relative gaps fall below a threshold. The public audit below reports componentwise BH; the block rule is the fallback when $g_j<\delta$.

\paragraph{Algorithm 2: Gap-based block test.}
\begin{enumerate}[leftmargin=*,itemsep=2pt]
\item Form ordered singular values $s_1,\ldots,s_r$ and define relative gaps $g_j=(s_j-s_{j+1})/\max(s_j,\varepsilon)$.
\item Starting at $j=1$, group adjacent components into a block whenever $g_j<\delta$; the default reporting threshold is $\delta=0.10$ and should be sensitivity-checked.
\item For a block $B=\{a,\ldots,b\}$, use a joint statistic such as $\sum_{j=a}^b s_j^2$ or a likelihood-inspired spectral-gap score.
\item Calibrate the block statistic using the same null constructor and the same deflation history used for the observed adapter.
\item Report block-level $p$-values and retained block masks. If block and componentwise decisions disagree, report both and treat the block result as the conservative block-level decision.
\end{enumerate}

\section{Synthetic Validation}
The numerical implementation uses log-domain summation of \eqref{eq:chisq} and selects $K$ until the dimension-aware bound \eqref{eq:tail} is below the requested tolerance. Null quantiles and powers below use independent random streams, exact dense singular values, and Wilson intervals for reported false-positive rates. Table~\ref{tab:ops} gives a compact synthetic calibration experiment for $m=n=128$, $\sigma=1$, and spike strengths $(30,18,12,8.5)$ using 5000 independent null draws and 2000 independent alternative draws per spike. Entries are empirical false-positive rate/power; the calibrated $s_1$ threshold targets $\alpha=0.05$, and Wilson intervals for the asymptotic edge false-positive rate give $0.129$ with interval $[0.120,0.139]$.

\begin{table}[t]
\centering
\caption{Synthetic operating points, shown as empirical FPR/power.}
\label{tab:ops}
\arxivtablesetup
\begin{adjustbox}{max width=\linewidth}
\begin{tabular}{@{}c*{5}{c}@{}}
\toprule
$\theta$ & $s_1$ cal. & edge & edge$+2\sigma$ & energy cal. & fixed $\tau=25$ \\
\midrule
8.5 & .050/.089 & .129/.215 & .000/.000 & .050/.061 & .000/.000 \\
12  & .050/.314 & .129/.492 & .000/.000 & .050/.233 & .000/.000 \\
18  & .050/1.00 & .129/1.00 & .000/.716 & .050/.999 & .000/.531 \\
30  & .050/1.00 & .129/1.00 & .000/1.00 & .050/1.00 & .000/1.00 \\
\bottomrule
\end{tabular}
\end{adjustbox}
\end{table}

A second diagnostic applies the simulated nulls sequentially to a rank-four signal. Table~\ref{tab:component} reports the nominal residual dimensions $(m-j+1,n-j+1)$ used at each stage and BH retention at $q=0.10$. The retention frequencies characterize the stated sequential-null simulation. The deployment analogue is an empirical null drawn from held-out seeds, task-preserving randomizations, or another declared construction at each stage.

\begin{table}[t]
\centering
\caption{Componentwise rank-four diagnostic; BH retention at $q=0.10$.}
\label{tab:component}
\arxivtablesetup
\begin{adjustbox}{max width=\linewidth}
\begin{tabular}{@{}crrrr@{}}
\toprule
$j$ & $\theta_j$ & residual size & median $p_j$ & retention \\
\midrule
1 & 30.0 & $128\times128$ & .0014 & 1.000 \\
2 & 18.0 & $127\times127$ & .0014 & 1.000 \\
3 & 12.0 & $126\times126$ & .164  & .351 \\
4 & 8.5  & $125\times125$ & .758  & .020 \\
\midrule
\multicolumn{4}{@{}l}{Mean retained rank} & 2.37 \\
\bottomrule
\end{tabular}
\end{adjustbox}
\end{table}

\begin{figure*}[t]
\centering
\includegraphics[width=.88\textwidth]{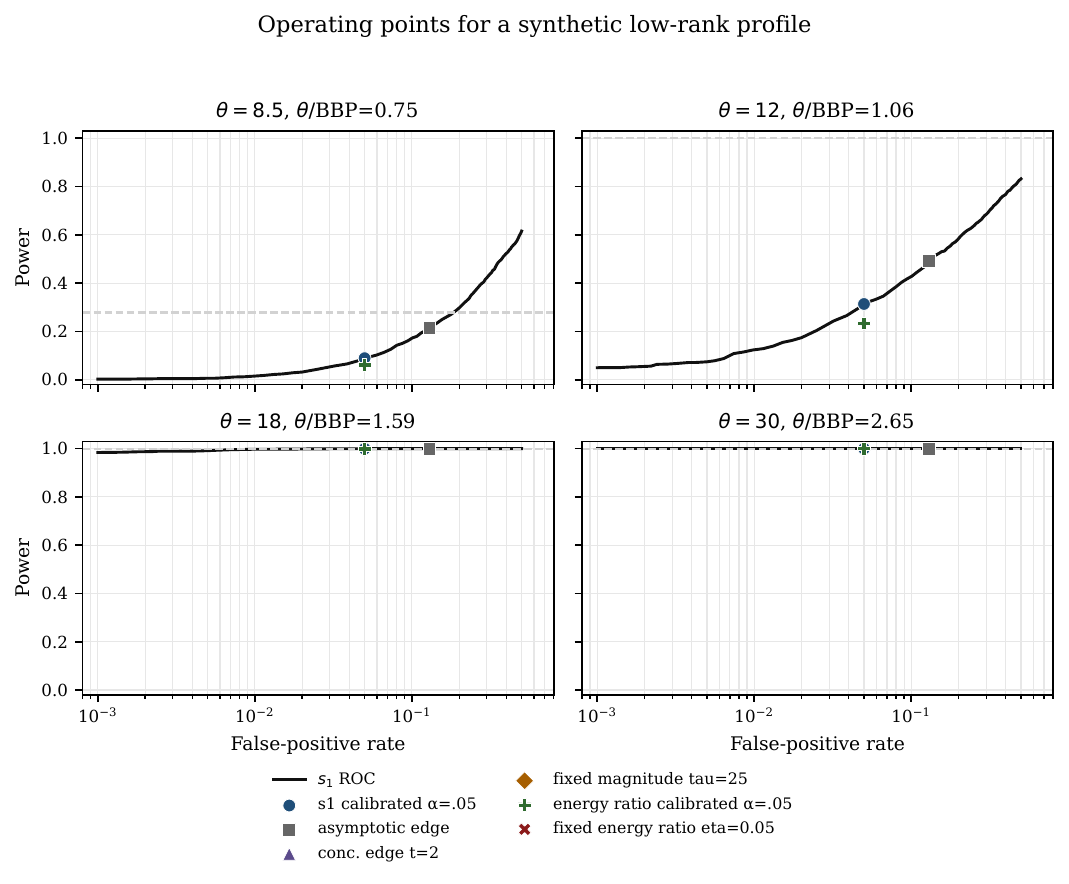}
\caption{Operating points for the same synthetic profile. The panels emphasize that calibration changes the reported false-positive rate, while strong spikes remain detectable.}
\label{fig:adapterdiag}
\end{figure*}

The synthetic results isolate three effects the reference model is built to expose: the asymptotic edge runs at a finite-sample level of $0.129$ compared with the nominal $0.05$, near-critical spikes have calibrated power well below one, and strong spikes are detected almost surely. Establishing analogous effects on trained adapters requires the empirical-null program in Section~\ref{sec:realadapter}.

\IfFileExists{data/deflation_calibration_summary.tex}{%
\begin{table}[t]
\centering
\caption{Null-only pipeline calibration. The iid Gaussian row checks the reference pipeline; the other rows are misspecification stress tests that motivate empirical nulls for deployment.}
\label{tab:deflation-calibration}
\arxivtablesetup
\begin{adjustbox}{max width=\linewidth}
\begin{tabular}{@{}lrrr@{}}
\toprule
Null used for evaluation & any retained & mean retained & KS dev. \\
\midrule
iid Gaussian & 0.094 & 0.163 & 0.030 \\
anisotropic rows & 0.982 & 3.621 & 0.842 \\
Student-$t_5$ & 0.437 & 1.127 & 0.366 \\
block-correlated & 0.904 & 0.904 & 0.886 \\
\bottomrule
\end{tabular}

\end{adjustbox}
\end{table}
}{}

\section{Public Adapter Checkpoints}
\label{sec:realadapter}
The empirical study uses a weights-only public-checkpoint corpus. Public PEFT checkpoints contain the LoRA factors needed for \eqref{eq:peft}, so the spectral audit can be performed directly from adapter weights. The implementation reconstructs spectra from pinned repositories and revisions and archives derived spectra, metadata, and checksums; third-party weights remain with their original sources.

The weights-only corpus spans RoBERTa, RoBERTa-large, XLM-RoBERTa,
FLAN-T5, Whisper, and LLaMA adapters (Table~\ref{tab:public-adapter-summary}). It includes a RoBERTa-large MRPC rank
sweep with nominal ranks $1,8,16,64$. For each module, the implementation compares
a primary factor-channel shuffle null with factor-sign and matched-Gaussian
sensitivity nulls. The primary null uses 999 samples per module, so the
smallest Monte Carlo $p$-value is $1/(B+1)=0.001$ and module-wise BH at $q=0.10$ is not
blocked by p-value resolution for ranks up to 64.

The resulting real-data pattern is consistent across families: most trained
modules have at least one empirically surprising component, but the calibrated
effective rank is far below the nominal rank. In the current corpus, module-wise
BH retains a median of one component per module, while a 95\% energy rule retains
a median of three components. From Table~\ref{tab:public-adapter-summary}, the MRPC
rank sweep is explicit: the $r=1$ adapter retains median module-wise BH rank $0$;
the $r=8$ and $r=16$ variants retain $1$; and the $r=64$ variant retains $2.5$.
Increasing nominal rank therefore has a nonlinear relationship with calibrated
effective rank. Whisper-large-v3-turbo (nominal rank $16$) retains median calibrated
rank $1$ under the primary null, illustrating that reconstruction energy and
statistical surprise answer different questions: energy retention typically
keeps many singular values on this family, while the calibrated decision remains sparse.
These results establish a corpus-level spectral finding from public adapter weights.
The broader corpus evidence is spectral; task metrics appear only in the RTE
measurement-path illustration below.

The measurement-path illustration uses the public RoBERTa RTE LoRA adapter on the
first $n=24$ GLUE-RTE validation examples under CPU inference, after verifying that
the saved adapter classifier head is loaded from the checkpoint. It compares the
full adapter, energy-95 rank truncation, corpus-level calibrated BH, and
module-wise calibrated BH after replacing LoRA factors by best rank-$k$
approximations of the dense adapter update.
Table~\ref{tab:downstream} reports the measured counts. On this slice the full
adapter and module-wise BH both achieve $18/24$ accuracy (Wilson $95\%$ interval
$[0.55,0.89]$), energy-95 achieves $16/24$ ($[0.47,0.82]$), and corpus-level BH
achieves $17/24$ ($[0.51,0.85]$). Module-wise BH keeps about $13\%$ of the LoRA
rank budget and matches the full-adapter count, while energy-95 keeps about $26\%$
and differs by two examples. With $n=24$, a two-example gap is not a significance
claim; the slice is a reproducible measurement path, not a utility study.

\IfFileExists{data/downstream_evaluation_summary.tex}{%
\begin{table}[t]
\centering
\caption{Measurement-path illustration on the public RoBERTa RTE LoRA adapter. Accuracy is computed on the first $n=24$ GLUE-RTE validation examples under CPU inference after verifying the saved adapter classifier head.}
\label{tab:downstream}
\arxivtablesetup
\begin{adjustbox}{max width=\linewidth}
\begin{tabular}{@{}lrrr@{}}
\toprule
Rank rule & accuracy & rank ratio & $n$ \\
\midrule
full & 0.750 & 1.000 & 24 \\
energy 95\% & 0.667 & 0.255 & 24 \\
calibrated BH & 0.708 & 0.120 & 24 \\
module BH & 0.750 & 0.130 & 24 \\
\bottomrule
\end{tabular}

\end{adjustbox}
\end{table}
}{}

\IfFileExists{data/real_adapter_summary.tex}{%
\begin{table*}[t]
\centering
\caption{Public-adapter summaries. Rows are real weights-only spectra from pinned PEFT adapters; retained rank is the median module-wise BH rank under the primary empirical null.}
\label{tab:public-adapter-summary}
\resizebox{\textwidth}{!}{\begin{tabular}{@{}llrrrl@{}}
\toprule
Architecture/task & source & modules & rank & med. module-BH rank & min $p$ \\
\midrule
FLAN-T5-large / GLUE CoLA & public checkpoint & 30 & 16 & 1.0 & 0.001 \\
FLAN-T5-large / GLUE MNLI & public checkpoint & 30 & 16 & 1.0 & 0.001 \\
FLAN-T5-large / GLUE MRPC & public checkpoint & 30 & 16 & 1.0 & 0.001 \\
FLAN-T5-large / GLUE QNLI & public checkpoint & 30 & 16 & 1.0 & 0.001 \\
FLAN-T5-large / GLUE QQP & public checkpoint & 30 & 16 & 1.0 & 0.001 \\
FLAN-T5-large / GLUE SST-2 & public checkpoint & 30 & 16 & 1.0 & 0.001 \\
FLAN-T5-large / GLUE STS-B & public checkpoint & 30 & 16 & 1.0 & 0.001 \\
FLAN-T5-large / GLUE WNLI & public checkpoint & 30 & 16 & 1.0 & 0.001 \\
FLAN-T5-large / SuperGLUE RTE & public checkpoint & 30 & 16 & 1.0 & 0.001 \\
FLAN-T5-large / SuperGLUE WiC & public checkpoint & 30 & 16 & 1.0 & 0.001 \\
LLaMA-7B / Alpaca instruction tuning & public checkpoint & 30 & 16 & 3.0 & 0.001 \\
RoBERTa-base / MNLI & public checkpoint & 24 & 8 & 1.0 & 0.001 \\
RoBERTa-base / MNLI partial-layer LoRA & public checkpoint & 12 & 8 & 2.0 & 0.001 \\
RoBERTa-base / RTE & public checkpoint & 24 & 8 & 1.0 & 0.001 \\
RoBERTa-base / RTE partial-layer LoRA & public checkpoint & 12 & 8 & 1.0 & 0.001 \\
RoBERTa-base / comma placement & public checkpoint & 24 & 16 & 4.0 & 0.001 \\
RoBERTa-base / comma placement finetuned & public checkpoint & 24 & 16 & 2.5 & 0.001 \\
RoBERTa-base / text classification & public checkpoint & 12 & 4 & 0.0 & 0.026 \\
RoBERTa-large / GLUE MRPC rank sweep r=1 & public checkpoint & 30 & 1 & 0.0 & 1 \\
RoBERTa-large / GLUE MRPC rank sweep r=16 & public checkpoint & 30 & 16 & 1.0 & 0.001 \\
RoBERTa-large / GLUE MRPC rank sweep r=64 & public checkpoint & 30 & 64 & 2.5 & 0.001 \\
RoBERTa-large / GLUE MRPC rank sweep r=8 & public checkpoint & 30 & 8 & 1.0 & 0.001 \\
Whisper-large-v3-turbo / disfluency ASR & public checkpoint & 30 & 16 & 1.0 & 0.001 \\
XLM-RoBERTa-base / Amharic news classification & public checkpoint & 24 & 8 & 1.0 & 0.001 \\
XLM-RoBERTa-base / LeNER-Br named entity recognition & public checkpoint & 24 & 16 & 2.0 & 0.001 \\
XLM-RoBERTa-base / language identification & public checkpoint & 24 & 8 & 1.0 & 0.001 \\
\bottomrule
\end{tabular}
}
\end{table*}
}{}

\IfFileExists{figures/fig7_adapter_schema_spectra.pdf}{%
\begin{figure*}[t]
\centering
\includegraphics[width=.92\textwidth]{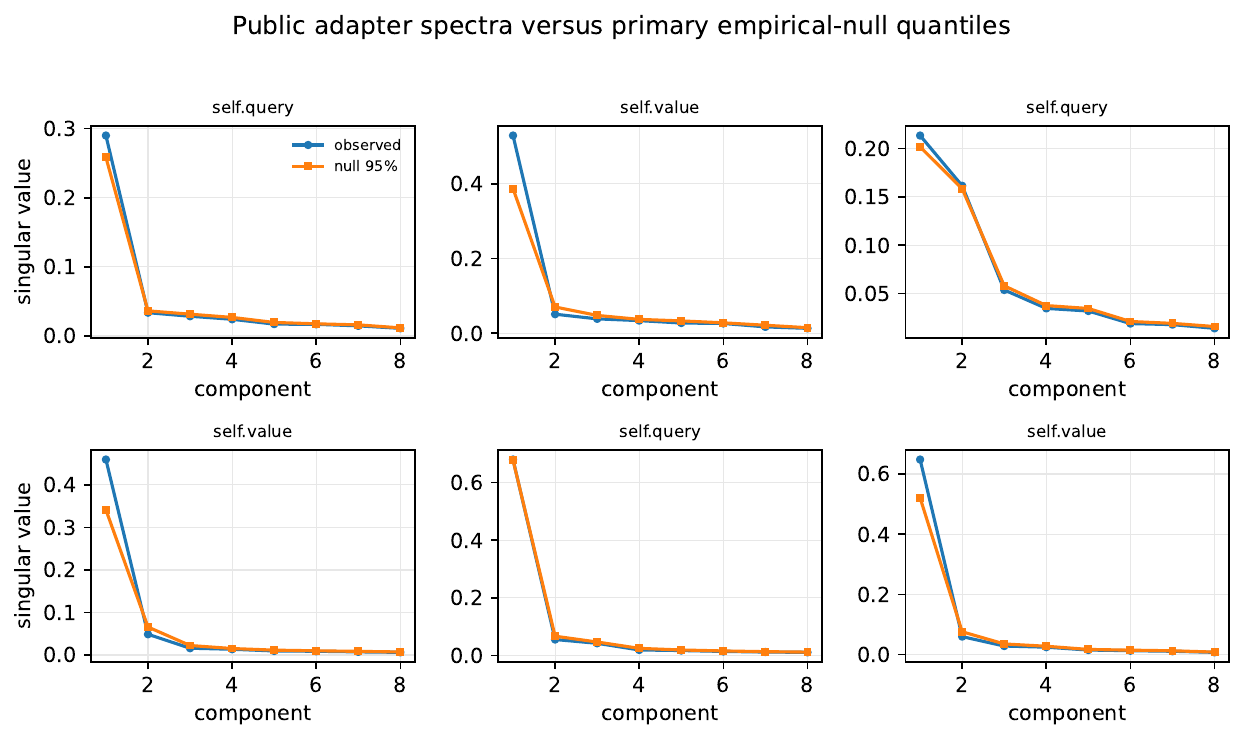}
\caption{Public adapter spectra versus primary empirical-null quantiles for selected module-level LoRA updates. The figure supports spectral calibration; task metrics are reported separately in Table~\ref{tab:downstream}.}
\label{fig:adapter-schema-spectra}
\end{figure*}
}{}

\section{Limitations}
The Gaussian/Haar experiment is a reference model for finite-sample calibration, not a claim that trained adapter residuals are isotropic Gaussians. Real adapter spectra may have anisotropic residuals, heavy tails, row/column variance heterogeneity, tensor structure, or optimizer-induced correlations. Equation~\eqref{eq:lecam} bounds attainable power for the stated rank-one reference experiment; deployment extends the bound through empirical nulls that reproduce the full deflation and block-testing pipeline. After adaptive deflation, BH output is an empirically calibrated reporting rule unless dependence conditions are proved. With $B=999$ Monte Carlo draws, the finest attainable $p$-value is $1/(B+1)=0.001$. Inconclusive residual diagnostics should be reported as abstentions rather than forced decisions. Calibrated effective rank is an input to engineering evaluation, not a substitute for task metrics, latency, or memory measurements on larger benchmarks.

\section{Reproducibility Statement}
The artifact regenerates synthetic tables, writes public-checkpoint adapter spectra from pinned repositories and revisions, checks numerical rank $\rank(\Delta W)\le r$, validates $p$-values in $[0,1]$, confirms retained counts do not exceed nominal rank, runs null-only pipeline calibration, and appends the measured RTE evaluation rows. Derived spectra, metadata, and checksums are archived; third-party weights remain with their original sources.

\section{Conclusion}
Adapter nominal rank can be separated from calibrated effective rank. The exact chi-square series provides a finite-sample Le Cam bound for the rank-one reference experiment, with an explicit remainder bound for numerical truncation, while the Laplace expansion explains why full-spectrum information can matter beyond the top singular value. For real adapters, empirical nulls extend the reference calculation to the full deflation and block-testing pipeline. The public-checkpoint study shows that calibrated effective rank is usually much smaller than nominal rank (for example, MRPC median BH ranks $0$, $1$, $1$, and $2.5$ at nominal ranks $1$, $8$, $16$, and $64$) and that energy retention and statistical surprise can disagree sharply. Those are the empirical findings that the finite-sample bound and the empirical-null workflow are designed to make precise.

Taken together, the findings align with spiked-matrix detection and outlier
theory \citep{benaych2012singular,elalaoui2018rectangular}, while the
full-spectrum and spectral-gap effects follow from the likelihood analysis
above. The misspecification and adaptive-deflation results motivate structured
empirical nulls, full-pipeline resampling, and multiplicity control
\citep{efron2010large,westfall1993resampling,benjamini1995controlling};
training dynamics and task-specific sensitivity remain possible explanations
for rank differences, not mechanisms established here.

\bibliographystyle{plainnat}
\bibliography{references}

\end{document}